\documentclass{article}
\usepackage{ijcai18}

\usepackage{times}
\usepackage{xcolor}
\usepackage{soul}
\usepackage[utf8]{inputenc}
\usepackage[small]{caption}

\usepackage{amsmath,amsfonts,amsthm}
\usepackage{bm}
\usepackage{algorithm,algpseudocode}
\usepackage{graphicx}

\algnewcommand{\Inputs}[1]{%
  \State \textbf{Inputs:}
  \Statex \hspace*{\algorithmicindent}\parbox[t]{.8\linewidth}{\raggedright #1}
}
\algnewcommand{\Initialize}[1]{%
  \State \textbf{Initialize:}
  \Statex \hspace*{\algorithmicindent}\parbox[t]{.8\linewidth}{\raggedright #1}
}

\DeclareMathOperator*{\argmax}{arg\,max}
\DeclareMathOperator*{\argmin}{arg\,min}

\newtheorem{theorem}{Theorem}
\newtheorem{lemma}{Lemma}
\newtheorem{assumption}{Assumption}

\title{Multinomial Logit Bandit with Linear Utility Functions}

\author{
Mingdong Ou, 
Nan Li, 
Shenghuo Zhu, 
Rong Jin, 
\\ 
Alibaba Group, Hang Zhou, China \\
\{mingdong.omd, nanli.ln, shenghuo.zhu, jinrong.jr\}@alibaba-inc.com,
}

\begin{document}

\maketitle

%
%
%
\begin{abstract}
Multinomial logit bandit is a sequential subset selection problem which arises in many applications. In each round, the player selects a $K$-cardinality subset from $N$ candidate items, and receives a reward which is governed by a {\it multinomial logit} (MNL) choice model considering both item utility and substitution property among items. The player's objective is to dynamically learn the parameters of MNL model and maximize cumulative reward over a finite horizon $T$. 
This problem faces the exploration-exploitation dilemma, and the involved combinatorial nature makes it non-trivial. 
In recent years, there have developed some algorithms by exploiting specific characteristics of the MNL model, but all of them estimate the parameters of MNL model separately and incur a regret no better than $\tilde{O}\big(\sqrt{NT}\big)$ which is not preferred for large candidate set size $N$. 
In this paper, we consider the {\it linear utility} MNL choice model whose item utilities are represented as linear functions of $d$-dimension item features, and propose an algorithm, titled {\bf LUMB}, to exploit the underlying structure. It is proven that the proposed algorithm achieves $\tilde{O}\big(dK\sqrt{T}\big)$ regret which is free of candidate set size.  Experiments show the superiority of the proposed algorithm.
\end{abstract}

\section{Introduction}\label{sec:intro}
In traditional stochastic multi-armed bandit (MAB)~\cite{Bubeck12}, the player selects {\it one} from $N$ items and receives a reward corresponding to that item in each round. 
The objective is to maximize cumulative reward over a finite horizon of length $T$, or alternatively, minimize the regret relative to an oracle. Typically, algorithms are designed based on appropriate exploration-exploitation tradeoff
which allows the player to identify the best item through exploration whilst not spending too much on sub-optimal ones, and the family of upper confidence bound (UCB) algorithms~\cite{auer2002using,chu2011contextual} and Thompson sampling (TS) \cite{thompson1933,agrawal2012analysis} are representative examples. 
This paper studies a combinatorial variant of MAB, where in each round, the player offers a subset of cardinality $K$ to a user, and receives the reward associated with one of the items in the selected subset\footnote{This problem is known as dynamic assortment selection in the literature~\cite{caro2007dynamic,rusmevichientong2010dynamic}, where the selected subset of items forms an assortment.}.
The player faces the problem of determining which subset of items to present to users who arrive sequentially, whilst not knowing user preference. Similar to MAB, we need to solve the exploration-exploitation tradeoff in this problem. However, a naive translation of this problem to MAB is prohibitive, since the number of possible $K$-cardinality subsets is exponentially large and cannot be efficiently explored within a reasonable sized time horizon. To tackle this issue, different strategies have been proposed in the literature, {\it e.g.}, \cite{kveton2015tight,Paul16,agrawal2017mnl}.

In recent literature, the {\it multinomial logit} (MNL) choice model~\cite{luce2005individual,plackett1975analysis} which is widely used in economics is utilized to model user choice behavior by specifying the probability that a user selects an item given the offered set, and above exploration-exploitation problem is referred as the MNL-Bandit problem~\cite{rusmevichientong2010dynamic,agrawal2017mnl,agrawal2017thompson,cheung2017assortment}. Unlike other combinatorial bandit problems, MNL-Bandit problem considers the substitution property among items and leads to non-monotonic reward function. By exploiting specific characteristics of the MNL model, UCB-style~\cite{agrawal2017mnl} and TS-style~\cite{agrawal2017thompson} algorithms have been developed to dynamically learn the parameters of the MNL model which are a priori unknown, achieving a regret of $\tilde{O}\big(\sqrt{NT}\big)$ under a mild assumption. Also, a lower regret bound is established in~\cite{agrawal2017mnl}, by showing that any algorithm based on the MNL choice model must incur a regret of $\Omega\big(\sqrt{NT/K}\big)$. It is easy to find that the regret depends on the number of candidate items $N$, making them less preferred in many large scale applications such as online advertising.

In this paper, we use the {\it linear utility} MNL choice model (formulated in Section~\ref{sec:formulation}) to model user choice behavior given a set of items, rather than traditional MNL model. Specifically, it is assumed that each item in candidate set is described by a $d$-dimension feature vector, and item utilities of the MNL model can be formulated as linear functions of item features. Based on this, the problem of estimating item utilities (\textit{i.e.}, parameters of the MNL model) is changed to estimating underlying model parameters of linear functions. Since the number of parameters is irrelevant with the number of items, it is possible to achieve more efficient solution when the number of items is large. 
By taking the UCB approach, we propose an algorithm, titled {LUMB} (which is short for {Linear Utility MNL-Bandit}), to dynamically learn the parameters and narrow the regret. The main contributions of this work include:
\begin{itemize}
\item To the best of our knowledge, this is the first to use linear utility MNL choice model in sequential subset selection problem. Also, an UCB-style algorithm LUMB is proposed to learn the model parameters dynamically. 
\item An upper regret bound $O\big(dK\sqrt{T}\left(\log{T}\right)^2\big)$ is established for the proposed LUMB algorithm. This regret bound is free of candidate item set size, which means that LUMB can be applied to large item set.
\item Empirical studies demonstrate the superiority of the proposed LUMB algorithm over existing algorithms.
\end{itemize}

The rest of this paper is organized as follows. Section~\ref{sec:related} briefly introduces related work. Section~\ref{sec:formulation} and \ref{sec:algo} present problem formulation and the LUMB algorithm. Section~\ref{sec:regret} establishes regret analysis. Section~\ref{sec:exp} summarizes the experiments, and Section~\ref{sec:conclusion} concludes this work with future directions.

\section{Related Work}\label{sec:related}
Classical bandit algorithms aim to find the best arm with exploration-exploitation strategy. \citeauthor{auer2002using}~[\citeyear{auer2002using}] first proposes a UCB approach in linear payoff setting. \citeauthor{dani2008stochastic}~[\citeyear{dani2008stochastic}] and \citeauthor{abbasi2011improved}~[\citeyear{abbasi2011improved}] propose improved algorithms which bound the linear parameters directly.  \citeauthor{agrawal2013thompson}~[\citeyear{agrawal2013thompson}] propose a Thompson sampling approach. However, because the reward of a subset is not a linear function of item features in the subset, these works cannot be directly applied to our problem.

Another class of bandit works related to our work is combinatorial bandit where the player selects a subset of arms and receive a collective reward in each round. Researchers study the problem mainly on two settings, stochastic setting~\cite{gai2012combinatorial,russo2014learning,kveton2015tight} and adversarial setting~\cite{cesa2012combinatorial,audibert2013regret}. \citeauthor{gai2012combinatorial}~[\citeyear{gai2012combinatorial}] first learn the problem in linear reward setting and \citeauthor{kveton2015tight}~[\citeyear{kveton2015tight}] prove a tight regret bound. It is generalized to non-linear rewards in~\cite{chen2016combinatorial,chen2013combinatorial}. \citeauthor{wen2015efficient}~[\citeyear{wen2015efficient}] and \citeauthor{wang2017efficient}~[\citeyear{wang2017efficient}] propose contextual algorithms which can handle large item sets. However, these works imply that the reward is monotonic which is not satisfied in MNL-Bandit (Section~\ref{sec:formulation}). In practice, as clarified in~\cite{cheung2017assortment}, the low-reward item may divert the attention of user and lead to lower subset reward.

\citeauthor{rusmevichientong2010dynamic}~[\citeyear{rusmevichientong2010dynamic}] solve the MNL-Bandit problem and achieve instance-dependent upper regret bound $O\left(logT\right)$, and \citeauthor{saure2013optimal}~[\citeyear{saure2013optimal}] extend to a wider class of choice models. \citeauthor{agrawal2017mnl}~[\citeyear{agrawal2017mnl}] propose a UCB approach and achieve instance-independent upper regret bound $O\left(\sqrt{NT}\right)$. \citeauthor{agrawal2017mnl}~[\citeyear{agrawal2017mnl}] propose a Thompson sampling approach with better empirical performance. Recently, some works begin to study variants of classic MNL-Bandit problem. Some works learn the personalized MNL-Bandit problem~\cite{kallus2016dynamic,bernstein2017dynamic,golrezaei2014real}.  \citeauthor{cheung2017assortment}~[\citeyear{cheung2017assortment}] learn the problem with resource constraints. However, as clarified in Section~\ref{sec:intro}, these works model item utility separately which is not feasible for large item sets.

\section{Problem Formulation}\label{sec:formulation}
Suppose there is a candidate item set, $\mathcal{S}=\{1,\cdots,N\}$, to offer to users. Each item $i$ corresponds to a reward $r_i\geq 0$ and a feature vector $\mathbf{x}_i \in \mathbf{R}^d$. Let $\mathbf{X}=[\mathbf{x}_1,\cdots,\mathbf{x}_N]$ be the feature matrix. In MNL-Bandit, at each time step $t$, the player selects a subset $\overline{S}_t\in\mathcal{C}\left(\mathcal{S}\right)=\{S|S\subseteq\mathcal{S},|S|\leq K\}$ and observes the user choice $c_t\in \overline{S}_t\cup \{0\}$, where $c_t=0$ represents that the user chooses nothing from $\overline{S}_t$. The objective is to design a bandit policy to approximately maximize the expected cumulative reward of chosen items, {\it i.e.}, $\mathbb{E}\left(\sum_{t}{r_{c_t}}\right)$. 

According to above setting, in each time step $t$, a bandit policy is only allowed to exploit the item features $\{\mathbf{x}_1, \cdots,\mathbf{x}_N\}$, historical selected subsets, $\{\overline{S}_\tau|\tau<t\}$ and the user choice feedbacks, $\{c_\tau|\tau<t\}$.

The user choice follows the MNL model~\cite{luce2005individual,plackett1975analysis}.  MNL assumes substitution property among items in a selected subset. Specifically, for item $i$, the larger the utilities of the other items, the smaller the chosen probability of item $i$ is. The choice probability follows a multinomial distribution,
\begin{align}
p_i\left(\overline{S}_t, \mathbf{v}\right)=
\begin{cases}
\dfrac{v_i}{1+\sum_{j\in \overline{S}_t}{v_j}}, & i \in \overline{S}_t \cr \dfrac{1}{1+\sum_{j\in \overline{S}_t}{v_j}}\ , & i=0 \cr 0, & \text{otherwise}\ ,
\end{cases}
\end{align}
where $v_i\geq 0$ is item utility which is a priori unknown to the player. The choice probability of item $i$ in selected subset, $\overline{S}_t$, is linear with its utility, $v_i$. Besides, it is possible that nothing is chosen which is realized by adding a virtual item $0$ with utility $1$.
With the MNL model, we have the expected reward under given utility vector, $\mathbf{v}=[v_1,\cdots,v_N]$,
\begin{align}
R\left(\overline{S}_t,\mathbf{v}\right)=\sum\nolimits_{i\in \overline{S}_t}p_i\left(\overline{S}_t,\mathbf{v}\right)r_i=\dfrac{\sum_{i\in \overline{S}_t}{v_ir_i}}{1+\sum_{i\in \overline{S}_t}{v_i}}\ . \label{equ:reward}
\end{align}
Note that the expected reward is non-monotonic, that is, both the addition of low-reward item to selected subset and increment on utility of low-reward item may lead to lower reward.
The expected cumulative reward is
\begin{align}\label{equ:cumrw}
\mathbb{E}\left(\sum\nolimits_{t}{r_{c_t}}\right)=\sum\nolimits_{t}{\mathbb{E}\left(r_{c_t}\right)}=\sum\nolimits_{t}{\mathbb{E}\left(R\left(\overline{S}_t,\mathbf{v}\right)\right)}\ .
\end{align}
Since direct analysis of (\ref{equ:cumrw}) is not tractable when $\mathbf{v}$ is unknown, we analyze the regret instead,
\begin{align}
Reg\left(T,\mathbf{v}\right)=\sum\nolimits_{t=1}^{T}{\left(R\left(S^*,\mathbf{v}\right)-\mathbb{E}\left(R\left(\overline{S}_t,\mathbf{v}\right)\right)\right)}\ , \label{equ:regret}
\end{align}
where $T$ is the length of time horizon and $S^*$ is the optimal subset,
\begin{align}
S^*=\displaystyle \mathop{\argmax}_{S\in\mathcal{C}\left(\mathcal{S}\right)}{R\left(S, \mathbf{v}\right)} \nonumber\ .
\end{align}
Naturally, the objective is to approximately minimize the expected cumulative regret, $Reg\left(T,\mathbf{v}\right)$, with appropriate bandit policy. Especially, after enough time steps, an appropriate solution should almost achieve the subsets with highest reward, which implies that the cumulative regret, $Reg\left(T,\mathbf{v}\right)$, should be sub-linear with $T$. As each item corresponds to an utility which needs to be estimated separately, this makes the lower cumulative regret bound relevant with item number and will be not feasible for large item sets.

Therefore, linear item utility is introduced where item utility is a linear function of item feature,
\begin{align}
v_i=\bm{\theta}^{*\top}\mathbf{x}_i\ ,
\end{align}
where $\bm{\theta}^{*}$ is a linear parameter vector unknown to the player. Thus, estimating item utilities will be changed to estimating utility function parameters which can exploit the correlation between items on features, then it is potential to achieve regret bound free of item number, $N$.

\section{Algorithm}\label{sec:algo}
We propose an algorithm, called \emph{Linear Utility MNL-Bandit} (LUMB), which proposes a UCB approach to sequentially estimate linear utility function and approach highest reward. LUMB first estimates the linear parameters of utility function, then constructs the UCB of item utility and subset reward, finally offers the subset with highest reward UCB. Algorithm \ref{alg:LUMB} clarifies the detail of LUMB.

\subsection{Estimation of Linear Utility Function} 
As the choice probability of an item is non-linear with the parameters of utility function, it is difficult to estimate the parameters directly with user choice feedback. Instead, we split the time horizon into epochs like~\cite{agrawal2017mnl}. Let $L$ be the number of epochs. In each epoch $l$, the selected subset, $S_l\in\mathcal{C}\left(\mathcal{S}\right)$, is offered repeatedly until that the user chooses nothing from the offered subset. Then, we can obtain the chosen times of each item $i \in S_l$,
\begin{align}
&\hat{v}_{i,l}=\sum\nolimits_{t\in\mathcal{E}_l}{\mathbb{I}\left(c_t=i\right)}  \label{equ:cins} \ ,\\
\text{s.t.}\quad&\mathbb{I}\left(c_t=i\right)=
\begin{cases}
1, & c_t=i \cr 0, & c_t\neq i \ ,
\end{cases} \nonumber
\end{align}
where $\mathcal{E}_l$ is the set of time steps in epoch $l$, $c_t$ is the chosen item in time step $t$.

It can be proven that $\mathbb{E}\left(\hat{v}_{i,l}\right)=v_i$ (Lemma \ref{lem:ctimes}), which means the empirical average of $\hat{v}_{i,l}$ is almost equal to real item utility and irrelevant to other items in the subset. Thus, the estimation of utility function can be simply formulated as a linear regression which directly approaches empirical samples of $\hat{v}_{i,l}$. Specifically,
\begin{align}
\bm{\theta_l}=\displaystyle \mathop{\argmin}_{\bm{\theta}}{\sum_{\tau\leq l}{\sum_{i\in S_\tau}{\|\bm{\theta}^\top\mathbf{x}_i-\hat{v}_{i,l}\|_2^2}}+\lambda\|\bm{\theta}\|_2^2} \nonumber \ ,
\end{align}
where $\lambda$ is a constant regularization coefficient.  Then, we can obtain close-form solution
\begin{align}
\mathbf{A}_l &= \lambda\mathbf{I}_d+\sum\nolimits_{\tau\leq l}{\sum\nolimits_{i\in S_\tau}{\mathbf{x}_i\mathbf{x}_i^\top}} \ ,\\
\mathbf{b}_l &= \sum\nolimits_{\tau\leq l}{\sum\nolimits_{i\in S_\tau}{\hat{v}_{i,\tau}\mathbf{x}_i}} \ ,\\
\theta_l&=\mathbf{A}_l^{-1}\mathbf{b}_l \ ,
\end{align}
where $\mathbf{I}_d$ is a $d$-by-$d$ identity matrix. 

\begin{algorithm}[t]
\caption{Linear Utility MNL-Bandit}
\label{alg:LUMB}
\begin{algorithmic}[1]
\Inputs{
$\mathcal{S}=\{1,\cdots,N\}$, $\mathbf{X}=[\mathbf{x}_1,\cdots,\mathbf{x}_N]$,  $\mathbf{r}=[r_1,\cdots,r_N]$
}
\Initialize{
$\theta_0\gets[0]^d$, $\mathbf{A_0}\gets\mathbf{I}_d$, $\mathbf{b_0}\gets[0]^d$,\\
$v_{i,0}^{\rm UCB}\gets\dfrac{\sqrt{2}+\alpha}{\lambda}\|\mathbf{x}_i\|,\forall i\leq N$\\
$t\gets 1$, $l\gets 1$, $\mathcal{E}_1\gets\emptyset$, $c_0\gets 0$
}
\Repeat
\If{$c_{t-1}=0$}
\State Compute $S_l\gets\displaystyle \mathop{\argmax}_{S\in\mathcal{S}}{R\left(S,\mathbf{v}_{l-1}^{\rm UCB}\right)}$
\EndIf
\State Offer subset $S_l$, observe the user choice $c_t$.
\If {$c_t=0$} 
\For{$i\in S_l$}
\State compute $\hat{v}_{i,l}\gets\sum_{t\in \mathcal{E}_l}{\mathbb{I}\left(c_t=i\right)}$
\EndFor
\State update $\mathbf{b}_{l}\gets\mathbf{b}_{l-1}+\sum_{i\in S_l}{\hat{v}_{i,l}\mathbf{x}_i}$
\State update $\mathbf{A}_{l}\gets\mathbf{A}_{l-1}+\sum_{i\in S_l}{\mathbf{x}_i\mathbf{x}_i^\top}$
\State update $\theta_{l}\gets\mathbf{A}_{l}^{-1}\mathbf{b}_{l}$
\For{$i\in \mathcal{S}$}
\State $v_{i,l}^{\rm UCB}\gets\theta_{l}^\top\mathbf{x}_i+\left(\sqrt{2}+\alpha\right)\sqrt{\mathbf{x}_i^\top\mathbf{A}_l^{-1}\mathbf{x}_i}$
\EndFor
\State $l\gets l+1$\
\State $\mathcal{E}_l\gets\emptyset$
\Else
\State $\mathcal{E}_l\gets\mathcal{E}_l\cup t$
\EndIf
\State $t\gets t+1$
\Until{$t<T$}
\end{algorithmic}
\end{algorithm}

\subsection{Construction of Upper Confidence Bound}
We construct the UCB of item utility, which is proven in Lemma~\ref{lem:ucb}, as
\begin{align}
&v_{i,l}^{\rm UCB}=\theta_l^\top\mathbf{x}_i+\left(\sqrt{2}+\alpha\right)\sigma_{i,l}\label{equ:vucb} \ , \\
\text{s.t.}\quad& \sigma_{i,l}=\sqrt{\mathbf{x}_i^\top\mathbf{A}_l^{-1}\mathbf{x}_i}  \nonumber
\end{align}
where $\alpha$ is constant. Let $\mathbf{v}_l^{\rm UCB}=[v_{1,l}^{\rm UCB},\cdots,v_{N,l}^{\rm UCB}]$.
Then, the UCB of the highest reward, $R\left(S^*,\mathbf{v}\right)$, is constructed as the highest reward with $\mathbf{v}_l^{\rm UCB}$ (Lemma~\ref{lem:reward}). The corresponding subset is
\begin{align}
S_{l+1}=\displaystyle \mathop{\argmax}_{S\in\mathcal{C}\left(\mathcal{S}\right)}{R\left(S, \mathbf{v}_l^{\rm UCB}\right)}\ ,
\end{align}
and we offer the subset $S_{l+1}$ in epoch $l+1$.

It is hard to get $S_{l+1}$ by directly solving the above optimization problem. According to \cite{davis2013assortment}, the above optimization problem can be translated to a linear program problem,
\begin{align}
\max\quad&{\sum\nolimits_{i=1}^N{r_iw_i}} \ ,\\
\text{s.t.}\quad&w_0+\sum\nolimits_{i=1}^N{w_i} = 1,\quad\sum\nolimits_{i=1}^N{\dfrac{w_i}{v_i^{\rm UCB}}}\leq Kw_0 \ ,\nonumber \\
& \forall i\in\mathcal{S}, 0\leq\dfrac{w_i}{v_i^{\rm UCB}}\leq w_0\nonumber \ .
\end{align}
Then, $S_{l+1}=\{i|w_i>0, i>0\}$.

\section{Regret Analysis} \label{sec:regret}
In this section, we analyze the upper regret bound of Algorithm \ref{alg:LUMB} to theoretically identify the convergence performance. 
Without loss of generality, we first declare the assumption in the following analysis.
\begin{assumption}\label{asum:para}
$\forall i\in\mathcal{S}, r_i\leq 1, \|\mathbf{x}_i\|\leq 1, \|\bm{\theta}^*\|\leq 1$.
\end{assumption}
According to the assumption, we have that $\forall i\in\mathcal{S}, v_i\leq 1$. Moreover, we let $\lambda=1$. Then, we give the upper bound of regret in Theorem \ref{thm:regret} in advance which is proven in Section \ref{sec:ubr}. We can achieve result similar to Theorem \ref{thm:regret} when the parameters in Assumption~\ref{asum:para} are bounded by finite constants.  
\begin{theorem}\label{thm:regret}
Following the process in Algorithm~\ref{alg:LUMB}, let 
\begin{align}
\beta&=2\log_2{T}, \nonumber \\
\alpha&=\beta\sqrt{2\log{\left(2\sqrt{T}\left(1+\dfrac{T}{d}\right)^{d/2}\right)}}, \nonumber
\end{align}
then the upper bound of $Reg\left(T,\mathbf{v}\right)$ is 
\begin{align}
O\left(dK\sqrt{T}\left(\log{T}\right)^2\right) \ .
\end{align}
\end{theorem}
The proof is separated into three steps. We first prove the correctness of the constructed UCB of utility and the cumulative deviation between UCB of utility  and real utility which is sublinear with respect to $T$. Then we prove that the deviation between UCB of reward and real reward can be bounded by deviation of utility, finally the upper bound of cumulative regret can be proved by combining the above two results.

\subsection{Analysis of Utility}
We first prove the distribution of $\hat{v}_{i,l}$.
\begin{lemma}\label{lem:ctimes}
With the definition in Eq.~(\ref{equ:cins}), $\forall l\leq L, i\in S_l$, $\hat{v}_{i,l}$ follows geometric distribution, that is
\begin{align}
\mathbb{P}\left(\hat{v}_{i,l}=\beta\right)&=\dfrac{1}{1+v_i}\left(\dfrac{v_i}{1+v_i}\right)^\beta, \forall \beta\geq 0\ ,\\
\mathbb{E}\left(\hat{v}_{i,l}\right)&=v_i \ .
\end{align}
\end{lemma}
Proof is in Appendix~\ref{sec:vprop}.
According to above Lemma, the deviation between $\hat{v}_{i,l}$ and real utility is unbounded. This makes the prove of utility UCB difficult. Fortunately, the probability $\mathbb{P}\left(\hat{v}_{i,l}>\beta\right)$ decays exponentially when $\beta$ increases. We bound $\hat{v}_{i,l}$ in a relative small interval with high probability. Then, we can prove the utility UCB as below.
\begin{lemma}\label{lem:ucb}
With definition of $v_{i,l}^{\rm UCB}$ in Eq.~(\ref{equ:vucb}) and definition of $\hat{v}_{i,l}$ in Eq.~(\ref{equ:cins}), if $\beta\geq 2\log_2{T}, \forall \tau\leq l, j\in S_l, \hat{v}_{j,\tau}\leq \beta$,  then $\forall i\in\mathcal{S}, l\leq L$, 
\begin{align}
0\leq v_{i,l}^{\rm UCB}-v_i\leq 2\left(\sqrt{2}+\alpha\right)\sigma_{i,l} \ ,
\end{align}
with probability at least 
\begin{align}
1-\left(1+\dfrac{T}{d}\right)^{d/2}\exp{\left(-\dfrac{\alpha^2}{2\beta^2}\right)}\ . \nonumber
\end{align}
\end{lemma}
Below is a brief proof, and detailed proof is in Appendix~\ref{sec:ucbproof}.
\begin{proof}
Let $\Delta_{i,l}=|\bm{\theta}_l^\top\mathbf{x}_i-v_i|$, we just need to prove
\begin{align}
\Delta_{i,l}\leq\left(\sqrt{2}+\alpha\right)\sigma_{i,l} \ . \nonumber
\end{align}
According to Lemma \ref{lem:ctimes},  when $\hat{v}_{i,l}\leq\beta$, 
\begin{align}
&\mathbb{E}\left(\hat{v}_{i,l}-v_i\right)=-\left(1+\beta\right)\dfrac{p_i^{\beta+1}}{1-p_i^{\beta+1}} \ ,\nonumber\\
&\text{s.t.}\quad p_i=\dfrac{v_i}{1+v_i} \nonumber \ .
\end{align}
Note that the result of $\mathbb{E}\left(\hat{v}_{i,l}-v_i\right)$ is irrelevant with $l$. Let $\epsilon_{i}=\mathbb{E}\left(\hat{v}_{i,l}-v_i\right)$.
\begin{align}
\Delta_{i,l}&\leq|\mathbf{x}_i^\top\mathbf{A}_{l}^{-1}\sum_{\tau\leq l}\sum_{j\in S_\tau}{\mathbf{x}_{j}\left(\hat{v}_{j,\tau}-v_j-\epsilon_{j}\right)}| \nonumber \\
						       &\quad+|\mathbf{x}_i^\top\mathbf{A}_{l}^{-1}\sum_{\tau\leq l}\sum_{j\in S_\tau}{\mathbf{x}_{j}\epsilon_{j}}-\mathbf{x}_i^\top\mathbf{A}_{l}^{-1}\mathbf{\theta}^*| \nonumber \ .
\end{align}
We prove the bound of two parts respectively. Let $u_i=\hat{v}_{i,\tau}-v_i-\epsilon_{i}$, $\mathbf{s}_l=\sum_{\tau\leq l}\sum_{i\in S_\tau}{\mathbf{x}_{i}u_i}$, it is easy to prove that
\begin{align}
\mathbb{E}\left(\exp\left(\gamma u_i\right)\right)\leq\exp(\dfrac{\gamma^2\beta^2}{2}) \nonumber \ ,
\end{align}
then, with Lemma 9 in~\cite{abbasi2011improved}, we can prove that with probability 
\begin{align}
1-\left(1+\dfrac{T}{d}\right)^{d/2}\exp{\left(-\dfrac{\alpha^2}{2\beta^2}\right)}, \nonumber
\end{align} 
we have that 
\begin{align}
&|\mathbf{x}_i^\top\mathbf{A}_{l}^{-1}\sum_{\tau\leq l}\sum_{i\in S_\tau}{\mathbf{x}_{i}\left(\hat{v}_{i,\tau}-v_i-\epsilon_{i}\right)}| \leq \alpha\sigma_{i,l} \ . \label{equ:vexp}
\end{align}
Moreover, with Cauchy–Schwarz inequality, the other part is bounded as
\begin{align}
|\mathbf{x}_i^\top\mathbf{A}_{l}^{-1}\sum_{\tau\leq l}\sum_{j\in S_\tau}{\mathbf{x}_{j}\epsilon_{j}} -\mathbf{x}_i^\top\mathbf{A}_{l}^{-1}\mathbf{\theta}^*|\leq\sqrt{2}\sigma_{i,l} \ . \label{equ:vtail}
\end{align}
The lemma can be proven by combining Eq.~(\ref{equ:vexp}) and Eq.~(\ref{equ:vtail}).
\end{proof}
Moreover, we prove the bound of cumulative utility deviation.
\begin{lemma}\label{lem:sigma}
Following the process in Algorithm~\ref{alg:LUMB}, the cumulative deviation between utility UCB and real utility can be bounded as
\begin{align}
\sum\nolimits_{l=1}^{L}\sum\nolimits_{i\in S_l}{\sigma_{i,l}}\leq K\sqrt{70dL\log L}\ .
\end{align}
\end{lemma}
Proof is similar to the proof of Lemma 3 in \cite{chu2011contextual}. Because of the space limitation, the proof will be attached in a longer version.
Lemma~\ref{lem:sigma} shows that the bound of cumulative deviation is sub-linear with epoch number,  and the average deviation in each epoch will vanish after enough epochs.

\subsection{Analysis of Reward}
We first estimate the deviation between estimated reward and real reward of $S_l$ with the result of Lemma~\ref{lem:sigma}.
\begin{lemma}\label{lem:dreward}
In each epoch $l$ of Algorithm~\ref{alg:LUMB}, if $\forall i \in S_l$, 
\begin{align}
0\leq v_i^{\rm UCB}-v_i\leq 2\left(\sqrt{2}+\alpha\right)\sigma_{i,l}\ ,\nonumber
\end{align}
then the cumulative deviation between estimated reward and real reward of $S_l$ is
\begin{align}
\mathbb{E}\left(\sum_{t\in \mathcal{E}_l}{\left(R\left(S_l, \mathbf{v}_l^{\rm UCB}\right)-r_{c_t}\right)}\right)\leq 2\left(\sqrt{2}+\alpha\right)\sum_{i\in S_l}{r_i\sigma_{i,l}}\ .
\end{align}
\end{lemma}
Proof is in Appendix~\ref{sec:rewardproof}. This Lemma means that the deviation of subset reward is bounded by deviation of item utilities in the subset.

\begin{lemma}\label{lem:reward}
(Lemma 4.2 in~\cite{agrawal2017mnl})With the reward defined in Eq.~(\ref{equ:reward}), suppose there are two subsets, $\tilde{S}^{\rm UCB}$ and $\tilde{S}$,  that
\begin{align}
\tilde{S}^{\rm UCB}&=\displaystyle \mathop{\argmax}_{S\in\mathcal{C}\left(\mathcal{S}\right)}{R\left(S,\mathbf{v}^{\rm UCB}\right)} \nonumber \ ,\\
\tilde{S}&=\displaystyle \mathop{\argmax}_{S\in\mathcal{C}\left(\mathcal{S}\right)}{R\left(S,\mathbf{v}\right)} \nonumber \ .
\end{align}
If $\forall i\in \tilde{S}, v_{i}^{\rm UCB}\geq v_i$, then the rewards satisfy the inequality
\begin{align}
R\left(\tilde{S}, \mathbf{v}\right)\leq R\left(\tilde{S}, \mathbf{v}^{\rm UCB}\right)\leq R\left(\tilde{S}^{\rm UCB},\mathbf{v}^{\rm UCB}\right).
\end{align}
\end{lemma}
Brief proof is in Appendix~\ref{sec:rewardproof}. Lemma~\ref{lem:reward} shows that the estimated reward of subset $S_l$ is an upper bound of real highest reward. Then we can easily bound the regret in each epoch with Lemma~\ref{lem:dreward} and Lemma~\ref{lem:reward}.
\begin{lemma} \label{lem:eregret}
(Regret bound in a single epoch) In each epoch $l$ of Algorithm~\ref{alg:LUMB}, if $\forall i \in S_l\cup S^*$
\begin{align}
0\leq v_i^{\rm UCB}-v_i\leq 2\left(\sqrt{2}+\alpha\right)\sigma_{i,l} \ ,
\end{align}
then the regret of epoch $l$ is
\begin{align}
\mathbb{E}\left(\sum_{t\in \mathcal{E}_l}{\left(R\left(S^*, \mathbf{v}\right)-r_{c_t}\right)}\right)\leq 2\left(\sqrt{2}+\alpha\right)\sum_{i\in S_l}{r_i\sigma_{i,l}}\ .
\end{align}
\end{lemma}

\subsection{Upper Bound of Regret}\label{sec:ubr}
We first prove a more general version of Theorem~\ref{thm:regret} with parameters, $\alpha$ and $\beta$.
\begin{lemma}\label{lem:genregret}
Following the process in Algorithm~\ref{alg:LUMB}, if $\beta\geq 2\log_2{T}$, the cumulative regret, defined in Eq.~(\ref{equ:regret}), can be bounded,
\begin{align}
Reg\left(T,\mathbf{v}\right)\leq& 2TK\left(1+\dfrac{T}{d}\right)^{d/2}\exp{\left(-\dfrac{\alpha^2}{2\beta^2}\right)}+\dfrac{T^2K}{2^{\beta+2}}\nonumber\\
				&+2\left(\sqrt{2}+\alpha\right)K\sqrt{70dT\log T} \ .
\end{align}

\end{lemma}
\begin{proof}
We obtain the bound of regret respectively in two situations: the item utility inequality in Lemma~\ref{lem:ucb} is (or not) satisfied.
We model the event that the item utility inequality in Lemma~\ref{lem:ucb} is not satisfied as 
\begin{align}
\mathcal{A}_{l}=\mathcal{U}_l\cup\mathcal{B}_l,
\end{align}
\begin{align}
\text{s.t.}\quad&\mathcal{U}_l=\{v_{i,l}^{\rm UCB}>v_i+2\left(\sqrt{2}+\alpha\right)\sigma_{i,l}\nonumber \\ 
					&\qquad\quad or\quad v_{i,l}^{\rm UCB} < v_i, \exists i \in S_{l}\cup S^*\}, \nonumber\\
&\mathcal{B}_l=\{\hat{v}_{i,\tau}>\beta,\exists \tau\leq l, i\in S_\tau\}. \nonumber
\end{align}
Then, it is easy to bound the probability of $\mathcal{A}_l$,
\begin{align}\label{equ:prob_nucb}
\mathbb{P}\left(\mathcal{A}_l\right)\leq 2K\left(1+\dfrac{T}{d}\right)^{d/2}\exp{\left(-\dfrac{\alpha^2}{2\beta^2}\right)}+\dfrac{lK}{2^{\beta+2}} \ .
\end{align}
Let $\tilde{\mathcal{A}}_l$ be the complement of $\mathcal{A}_{l}$. Then, the regret can be splited into two parts, that is,
\begin{align}
Reg\left(T,\mathbf{v}\right)&=\mathbb{E}\left(\sum_{l=1}^{L}\sum_{t\in\mathcal{E}_l}{\mathbb{I}\left(\mathcal{A}_{l-1}\right)\left(R\left(S^*,\mathbf{v}\right)-r_{c_t}\right)}\right) \nonumber\\
			   &\quad+\mathbb{E}\left(\sum_{l=1}^{L}\sum_{t\in\mathcal{E}_l}{\mathbb{I}\left(\tilde{\mathcal{A}}_{l-1}\right)\left(R\left(S^*,\mathbf{v}\right)-r_{c_t}\right)}\right)\nonumber \ ,
\end{align}
where $\mathbb{I}\left(\mathcal{A}\right)$ is an indicator random variable whose value is $1$ when $\mathcal{A}$ happens, othewise $0$.
We first consider the situation that $\mathcal{A}_l$ happens, according to Equation~(\ref{equ:prob_nucb})

\begin{align}
&\mathbb{E}\left(\sum\nolimits_{l=1}^{L}\sum\nolimits_{t\in\mathcal{E}_l}{\mathbb{I}\left(\mathcal{A}_l\right)\left(R\left(S^*,\mathbf{v}\right)-r_{c_t}\right)}\right) \nonumber \\
&\le\mathbb{E}\left(\sum\nolimits_{l=1}^{L}\sum\nolimits_{t\in\mathcal{E}_l}{\mathbb{I}\left(\mathcal{A}_l\right)}\right) \nonumber \\
&=\mathbb{E}\left(\sum\nolimits_{l=1}^{L}\sum\nolimits_{t\in\mathcal{E}_l}{\mathbb{E}\left(\mathbb{I}\left(\mathcal{A}_l\right)\right)}\right) \nonumber \\
&=\mathbb{E}\left(\sum\nolimits_{l=1}^{L}\sum\nolimits_{t\in\mathcal{E}_l}{\mathbb{P}\left(\mathcal{A}_l\right)}\right) \nonumber \\
&\le\mathbb{E}\left(\sum\nolimits_{l=1}^{L}\sum\nolimits_{t\in\mathcal{E}_l}{1}\right)\left(2K\left(1+\dfrac{T}{d}\right)^{d/2}\exp{\left(-\dfrac{\alpha^2}{2\beta^2}\right)}+\dfrac{TK}{2^{\beta+2}}\right) \nonumber \\
&= 2TK\left(1+\dfrac{T}{d}\right)^{d/2}\exp{\left(-\dfrac{\alpha^2}{2\beta^2}\right)}+\dfrac{T^2K}{2^{\beta+2}} \ .\label{equ:nucbregret}
\end{align}
Then, we consider that $\mathcal{A}_l$ does not happen. According to Lemma~\ref{lem:eregret} and Lemma~\ref{lem:sigma},
\begin{align}
&\mathbb{E}\left(\sum\nolimits_{l=1}^{L}\sum\nolimits_{t\in\mathcal{E}_l}{\mathbb{I}\left(\tilde{\mathcal{A}}_{l-1}\right)\left(R\left(S^*,\mathbf{v}\right)-r_{c_t}\right)}\right) \nonumber \\
&\le\mathbb{E}\left(\sum\nolimits_{l=1}^{L}\sum\nolimits_{t\in\mathcal{E}_l}{\left(R\left(S^*,\mathbf{v}\right)-r_{c_t}\right)}\right) \nonumber \\
&\le2\left(\sqrt{2}+\alpha\right)\mathbb{E}\left(\sum\nolimits_{l=1}^{L}\sum\nolimits_{i\in S_l}{r_i\sigma_{i,l}}\right) \nonumber \\
&\le2\left(\sqrt{2}+\alpha\right)\mathbb{E}\left(\sum\nolimits_{l=1}^{L}\sum\nolimits_{i\in S_l}{\sigma_{i,l}}\right) \nonumber \\
&\leq2\left(\sqrt{2}+\alpha\right)K\sqrt{70dT\log{T}} \label{equ:ucbregret} \ .
\end{align}
Finally, we can finish the proof by adding Eq.~(\ref{equ:nucbregret}) and Eq.~(\ref{equ:ucbregret}).

\end{proof}
With Lemma~\ref{lem:genregret}, Theorem~\ref{thm:regret} can be proven by setting 
\begin{align}
\beta&=2\log_2{T} \nonumber \ , \\
\alpha&=\beta\sqrt{2\log{\left(2\sqrt{T}\left(1+\dfrac{T}{d}\right)^{d/2}\right)}} \ . \nonumber
\end{align}
As our method is in the similar framework of MNL-Bandit~\cite{agrawal2017mnl} whose lower bound is $\tilde{O}\left(\sqrt{T}\right)$, our regret bound matches the lower bound up to logarithmic terms with respect to $T$.

\section{Experiments}\label{sec:exp}
\begin{figure}[t]
\centering
\includegraphics[width=0.45\textwidth]{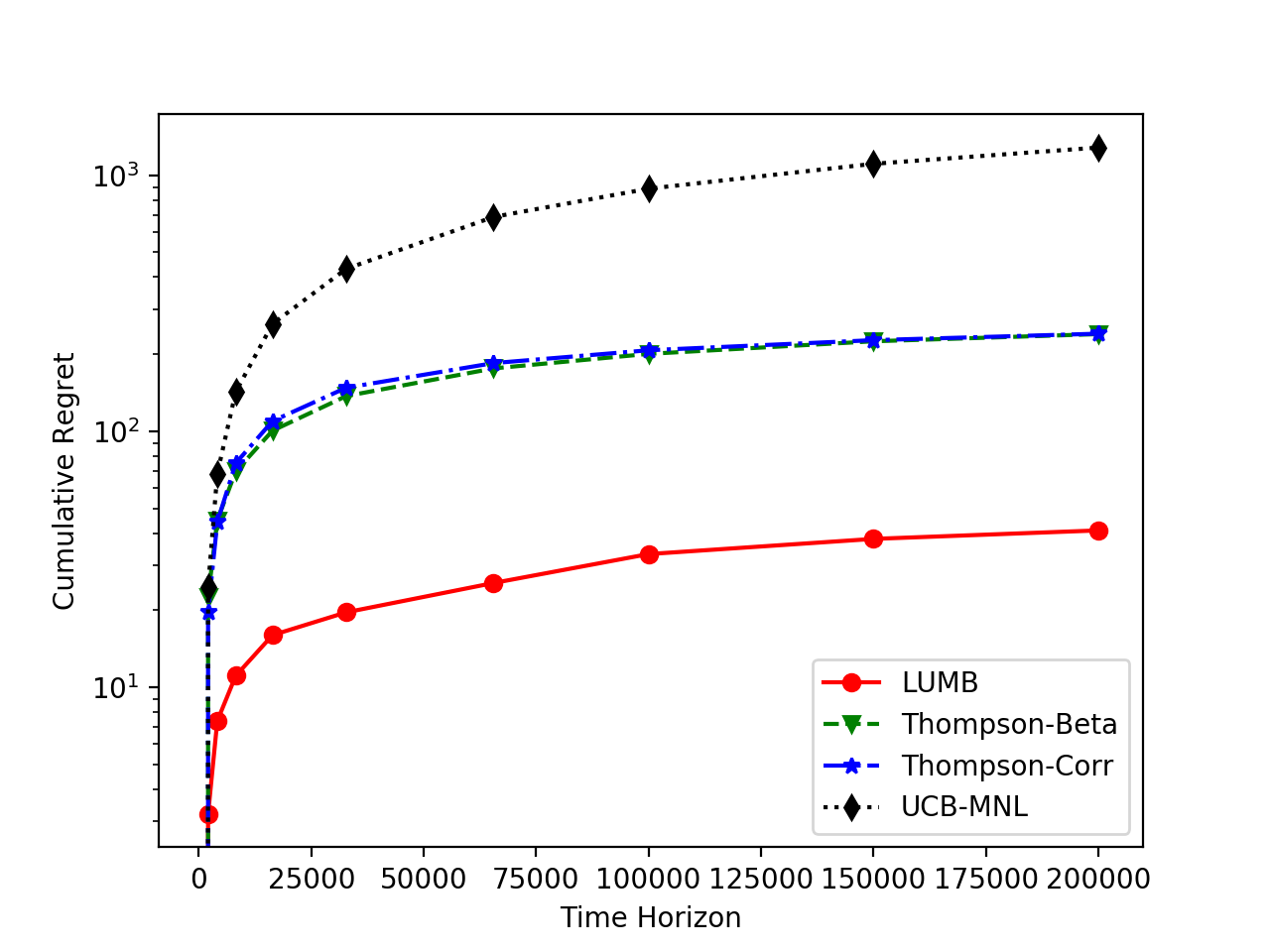}
\caption{Cumulative regret along time horizon.}
\label{fig:regret}
\end{figure}
In this section, we evaluate LUMB on synthetic data and compare to three existing alternative algorithms. We demonstrate the superiority of LUMB on cumulative regret. Moreover, we show that the estimated linear parameters of utility function and utilities will asymptotically converge to the real value.
\subsection{Setting}
The synthetic data is generated randomly. $N$ rewards are sampled from interval $(0,1]$ uniformly. $d$-dimension parameter vector of utility function, $\bm{\theta}^*$, is sampled from $[0,1]^d$ uniformly, then is normalized to $1$. $N$ $d$-dimension feature vectors are sampled from $[0,1]^d$ uniformly. To follow the experiment setting in \cite{agrawal2017thompson}, feature vectors are normalized so that item utilities distribute uniformly on $[0,1]$. Experiments are all performed on ten randomly generated data sets and the results show below are all average of results on these data sets.

Three alternative algorithms are compared:
\begin{itemize}
\item UCB-MNL~\cite{agrawal2017mnl}: This algorithm proposes a UCB approach with MNL choice model. 
\item Thompson-Beta~\cite{agrawal2017thompson}: This algorithm proposes a Thompson sampling approach with MNL choice model. 
\item Thompson-Corr~\cite{agrawal2017thompson}: This algorithm is a variant of Thompson-Beta which samples item utilities with correlated sampling.
\end{itemize}

\subsection{Results}
We conduct empirical experiments on synthetic data sets with $N=1000$,$d=10$.  Subset size $K$ is set to $10$.  Figure~\ref{fig:regret} shows the cumulative regret on the synthetic data sets, which is normalized by best reward, \textit{i.e.}, $Reg\left(t,\mathbf{v}\right)/R\left(S^*,\mathbf{v}\right)$. Note that the axis of cumulative regret is plot in a logarithm style for the convenience of observing the trend of LUMB regret on time horizon. The cumulative regrets increase slower when time step increases. Besides, we can see that the cumulative regret of LUMB is much smaller than the alternative algorithms through the time horizon.

\begin{figure}[!t]
\centering
\includegraphics[width=0.45\textwidth]{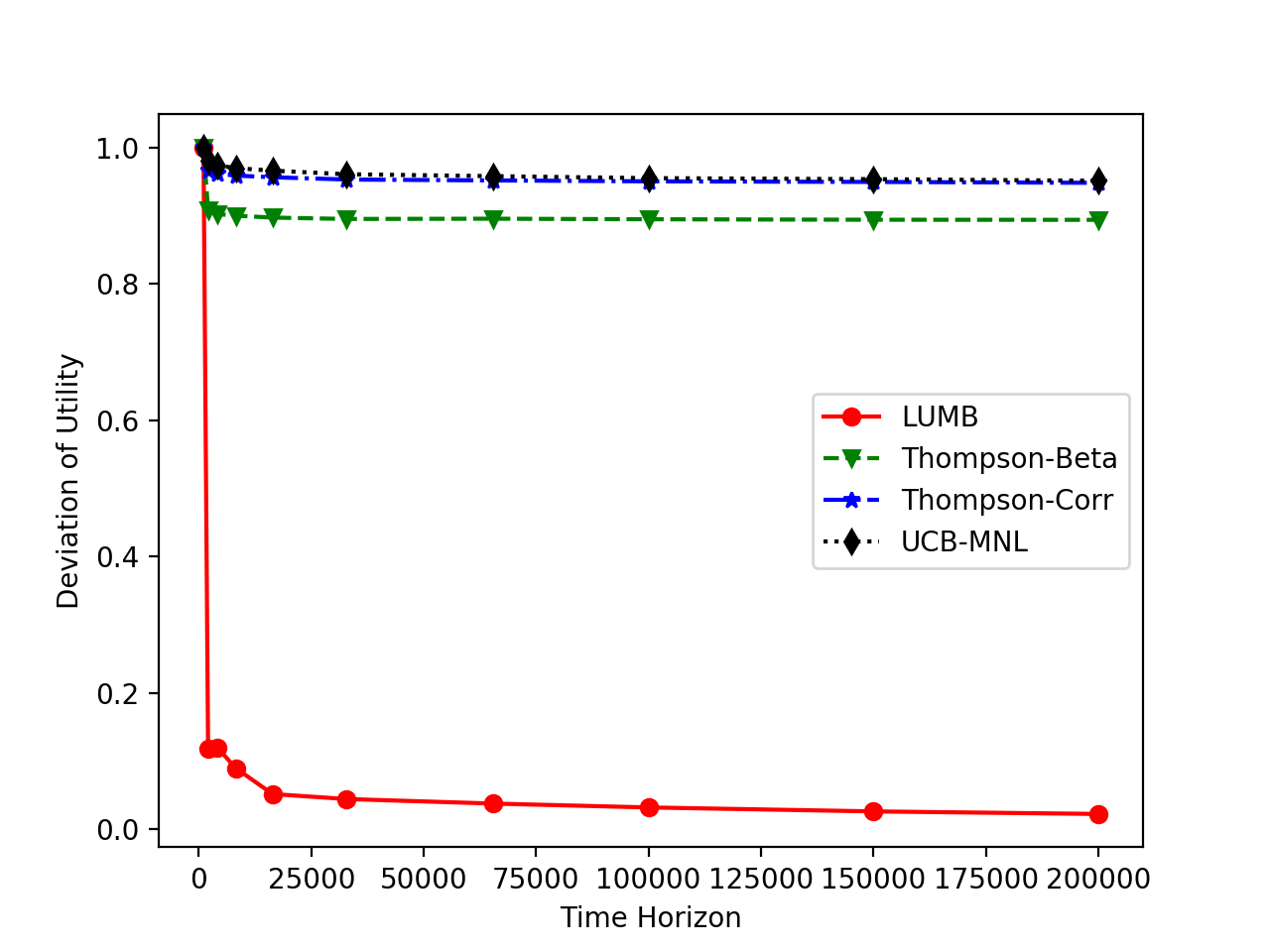}
\includegraphics[width=0.45\textwidth]{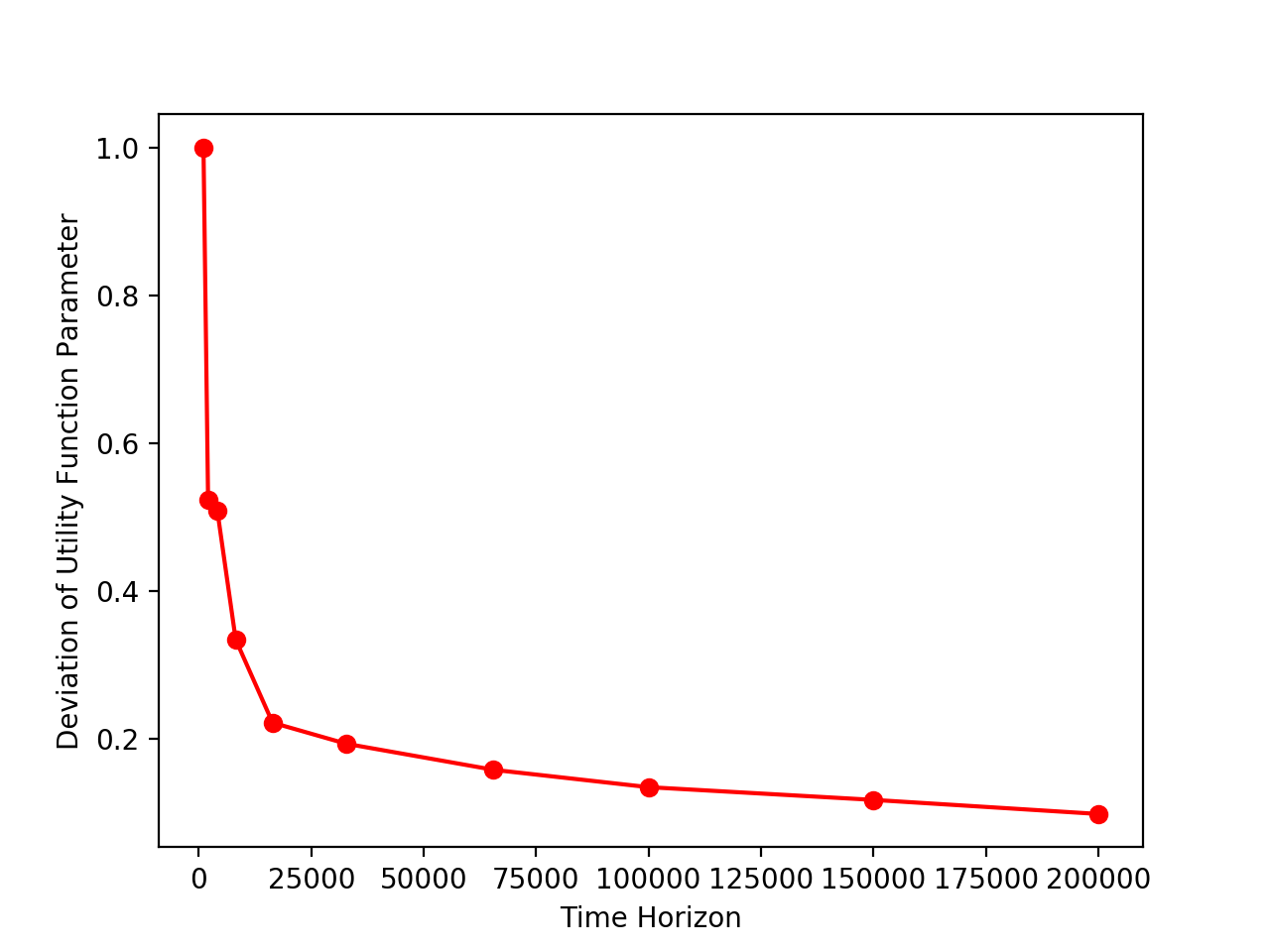}
\caption{Deviation of item utility vector (upper figure) and linear utility function parameter vector (lower figure) over time horizon. }
\label{fig:utility}
\end{figure}

We evaluate the convergence of utility vector on synthetic data sets with $N=1000$,$d=10$. Figure~\ref{fig:utility} shows the deviation between estimated mean utility and real utility, which is normalized by the norm of real utility, \textit{i.e.}, $\|\bm{\theta}_t^\top\mathbf{X}-\mathbf{v}\|/\|\mathbf{v}\|$. The deviation of LUMB decreases fast in the early stage and achieve smaller deviation compared to the alternative algorithms. 

Moreover, we evaluate the deviation of estimated linear parameter vector in Figure~\ref{fig:utility}. The deviation is normalized by the norm of real parameters, \textit{i.e.}, $\|\bm{\theta}_t-\bm{\theta}^*\|/\|\bm{\theta}^*\|$. Note that the deviation also decreases fast in the early stage and asymptotically converges to zero finally. This demonstrates that LUMB can correctly estimate the linear parameters.


\section{Conclusion}\label{sec:conclusion}
We study the sequential subset selection problem with linear utility MNL choice model, and propose a UCB-style algorithm, LUMB. Also, an upper regret bound, $O\left(dK\sqrt{T}\left(\log{T}\right)^2\right)$, is established, which is free of candidate item number. Experiments show the performance of LUMB. In the future work, we are interested in extending the idea to other choice models such as nested MNL.

\bibliographystyle{named}

\clearpage
\normalsize
\appendix
\section{Proof of Lemmas}

\subsection{Proof of Lemma~\ref{lem:ctimes}}\label{sec:vprop}
\begin{proof}
Let $V=\displaystyle \mathop{\sum}_{i\in S_{l}}{v_i}$
\begin{align}
&\mathcal{P}\left(\hat{v}_{i,l}=\beta\right) \nonumber \\
=&\sum_{n\geq\beta}\mathcal{P}\left(|\mathcal{E}_{l}|=n+1\right)\mathcal{P}\left(\hat{v}_{i,l}=\beta||\mathcal{E}_{l}|=n+1\right) \nonumber\\
=&\sum_{n\geq\beta}{\dfrac{1}{1+V}\left(\dfrac{V}{1+V}\right)^n{n \choose \beta}}\left(\dfrac{v_i}{V}\right)^\beta\left(\dfrac{V-v_i}{V}\right)^{n-\beta} \nonumber
\end{align}
Let $\beta=0$, we have $\mathcal{P}\left(\hat{v}_{i,l}=0\right)=\dfrac{1}{1+v_i}$.

With ${n \choose \beta}={n-1 \choose \beta}+{n-1 \choose \beta-1}$, we have 
\begin{align}
\mathcal{P}\left(\hat{v}_{i,l}=\beta\right)=\dfrac{v_i}{1+v_i}\mathcal{P}\left(\hat{v}_{i,l}=\beta-1\right) \nonumber \ .
\end{align}
Finally,
\begin{align}
\mathcal{P}\left(\hat{v}_{i,l}=\beta\right)&=\dfrac{1}{1+v_i}\left(\dfrac{v_i}{1+v_i}\right)^\beta \nonumber \ , \\
\mathbb{E}\left(\hat{v}_{i,l}\right)&=v_i \ . \nonumber
\end{align}
\end{proof}

\subsection{Proof of Lemma \ref{lem:ucb}}\label{sec:ucbproof}
\begin{lemma}\label{lem:vsubgau}
Suppose random variable $X$ satisfies that $a\leq X\leq b$ and $\mathbb{E}\left(X\right)=0$, then
\begin{align}
\mathbb{E}\left(e^{\lambda X}\right)\leq e^{\lambda^2\left(b-a\right)^2/2} \ .
\end{align}
\end{lemma}
\begin{proof}
Let $X^{'}$ be a random variable with same distribution as $X$ but independent of $X$. According to Jensen's inequality:
\begin{align}
\mathbb{E}_{X}\left(e^{\lambda X}\right) &= \mathbb{E}_X\left(e^{\lambda \left(X-\mathbb{E}_{X^{'}}\left(X^{'}\right)\right)}\right) \nonumber\\
						&\leq \mathbb{E}_{X,X^{'}}\left(e^{\lambda\left(X-X^{'}\right)}\right) \nonumber
\end{align}
As $X$ and $X^{'}$ are with same distribution, we have $\mathbb{E}_{X,X^{'}}\left(e^{\lambda\left(X-X^{'}\right)}\right)=\mathbb{E}_{X,X^{'}}\left(e^{\lambda\left(X^{'}-X\right)}\right)$. 
\begin{align}
&\mathbb{E}_{X,X^{'}}\left(e^{\lambda\left(X-X^{'}\right)}\right) \nonumber \\
=&\dfrac{1}{2}\mathbb{E}_{X,X^{'}}\left(e^{\lambda\left(X-X^{'}\right)}+e^{\lambda\left(X^{'}-X\right)}\right) \nonumber \\
=&\dfrac{1}{2}\mathbb{E}_{X,X^{'}}\left(\sum_{k=0}^{\infty}\dfrac{\lambda^k\left(X-X^{'}\right)^k}{k!}+\sum_{k=0}^{\infty}\dfrac{\lambda^k\left(X^{'}-X\right)^k}{k!}\right) \nonumber \\
=&\mathbb{E}_{X,X^{'}}\left(\sum_{k=0}^{\infty}\dfrac{\lambda^{2k}\left(X-X^{'}\right)^{2k}}{\left(2k\right)!}\right) \nonumber \\
\leq&\mathbb{E}_{X,X^{'}}\left(\sum_{k=0}^{\infty}\dfrac{\lambda^{2k}\left(X-X^{'}\right)^{2k}}{2^kk!}\right) \nonumber \\
\leq&\mathbb{E}_{X,X^{'}}\left(e^{\lambda^2\left(X-X^{'}\right)^2/2}\right) \nonumber \\
\leq&e^{\lambda^2\left(b-a\right)^2/2}
\end{align}
\end{proof}
We give the proof of Lemma~\ref{lem:ucb} below.
\begin{proof}
\begin{align}
v_{i,l}^{\rm UCB}-v_i=\left(\bm{\theta}_l^\top\mathbf{x}_i-v_i\right)+\left(\sqrt{2}+\alpha\right)\sigma_{i,l} \ .
\end{align}
We just need to prove the bound of $\Delta_{i,l}=|\bm{\theta}_l^\top\mathbf{x}_i-v_i|$.

According to Lemma \ref{lem:ctimes},  when $\hat{v}_{i,l}\leq\beta$, 
\begin{align}
&\mathbb{E}\left(\hat{v}_{i,l}-v_i\right)=-\left(1+\beta\right)\dfrac{p_i^{\beta+1}}{1-p_i^{\beta+1}} \ ,\\
&s.t.\quad p_i=\dfrac{v_i}{1+v_i} \nonumber \ .
\end{align}
According to Assumption~\ref{asum:para}, we can have 
\begin{align}
v_i={\theta^*}^\top\mathbf{x}_i\le 1\ ,\nonumber
\end{align}
then
\begin{align}
|\mathbb{E}\left(\hat{v}_{i,l}-v_i\right)|\le \dfrac{1+\beta}{2^\beta} \le \dfrac{1+\log_2{T}}{T^2}
\end{align}
Note that $\mathbb{E}\left(\hat{v}_{i,l}-v_i\right)$ is irrelevant with $l$ and will be very small when $T$ is large. Let $\epsilon_{i}=\mathbb{E}\left(\hat{v}_{i,l}-v_i\right)$.
\begin{align}
\Delta_{i,l}&=|\mathbf{x}_i^\top\mathbf{A}_{l}^{-1}\mathbf{b}_{l}-\mathbf{x}_i^\top\mathbf{A}_{l}^{-1}\mathbf{A}_{l}\mathbf{\theta}^*| \nonumber \\
						       &=|\mathbf{x}_i^\top\mathbf{A}_{l}^{-1}\sum_{\tau\leq l}\sum_{j\in S_\tau}{\mathbf{x}_{j}\hat{v}_{j,\tau}}\nonumber\\
						       &-\mathbf{x}_i^\top\mathbf{A}_l^{-1}\left(\mathbf{I}_d+\sum_{\tau\leq l}\sum_{j\in S_\tau}{\mathbf{x}_{j}\mathbf{x}_{j}}^\top\right)\bm{\theta}^*| \nonumber\\
						       &\leq|\mathbf{x}_i^\top\mathbf{A}_{l}^{-1}\sum_{\tau\leq l}\sum_{j\in S_\tau}{\mathbf{x}_{j}\left(\hat{v}_{j,\tau}-v_j-\epsilon_{j}\right)}| \nonumber \\
						       &\quad+|\mathbf{x}_i^\top\mathbf{A}_{l}^{-1}\sum_{\tau\leq l}\sum_{j\in S_\tau}{\mathbf{x}_{j}\epsilon_{j}}-\mathbf{x}_i^\top\mathbf{A}_{l}^{-1}\mathbf{\theta}^*| \nonumber \ .
\end{align}
We prove the bound of two parts respectively.

Let $\mathbf{s}_l=\sum_{\tau\leq l}\sum_{i\in S_\tau}{\mathbf{x}_{i}\left(\hat{v}_{i,\tau}-v_i-\epsilon_{i}\right)}$. With Lemma 9 in~\cite{abbasi2011improved} and Lemma~\ref{lem:vsubgau}, we can prove that with probability $1-\delta$,
\begin{align}
&|\mathbf{x}_i^\top\mathbf{A}_{l}^{-1}\sum_{\tau\leq l}\sum_{i\in S_\tau}{\mathbf{x}_{i}\left(\hat{v}_{i,\tau}-v_i-\epsilon_{i}\right)}|\nonumber \\
&\leq |\mathbf{x}_i^\top\mathbf{A}_l^{-1}\mathbf{x}_i||\mathbf{s}_l^\top\mathbf{A}_l^{-1}\mathbf{s}_l|\nonumber \\
&\leq \sigma_{i,l}\beta\sqrt{2\log\dfrac{\left(1+T/d\right)^{d/2}}{\delta}} \nonumber
\end{align}
Let $\alpha=\beta\sqrt{2\log\dfrac{\left(1+T/d\right)^{d/2}}{\delta}}$. We have that with probability $1-\left(1+\dfrac{t}{d}\right)^{d/2}\exp{\left(-\dfrac{\alpha^2}{2\beta^2}\right)}$, 
\begin{align}
&|\mathbf{x}_i^\top\mathbf{A}_{l}^{-1}\sum_{\tau\leq l}\sum_{i\in S_\tau}{\mathbf{x}_{i}\left(\hat{v}_{i,\tau}-v_i-\epsilon_{i}\right)}| \leq \alpha\sigma_{i,l} \ . \label{equ:vexp}
\end{align}
Moreover, when $K\le T$, we have that
\begin{align}
&|\mathbf{x}_i^\top\mathbf{A}_{l}^{-1}\sum_{\tau\leq l}\sum_{j\in S_\tau}{\mathbf{x}_{j}\epsilon_{j}} -\mathbf{x}_i^\top\mathbf{A}_{l}^{-1}\mathbf{\theta}^*|\nonumber\\
\leq&\sqrt{\mathbf{x}_i^\top\mathbf{A}_{l}^{-1}\left(\mathbf{I}_d+\sum_{\tau\leq l}\sum_{j\in S_l}{\mathbf{x}_{j}\mathbf{x}_{j}^\top}\right)\mathbf{A}_{l}^{-1}\mathbf{x}_i}\nonumber\\
&\quad\cdot\sqrt{\sum_{\tau\leq l}\sum_{j\in S_l}{\epsilon_{j}^2}+\bm{\theta}^{*\top}\bm{\theta}^*}\nonumber \\
\leq&\sqrt{2}\sigma_{i,l} \ . \label{equ:vtail}
\end{align}
Combine the result of (\ref{equ:vexp}) and (\ref{equ:vtail}), the lemma can be proven.
\end{proof}

\subsection{Proof of Lemma \ref{lem:dreward}}\label{sec:rewardproof}
\begin{proof}
As $\mathcal{E}_l$ follows geometric distribution, $\mathbb{E}\left(|\mathcal{E}_l|\right)=1+\sum_{i\in S_l}v_i$.
In the last time step in $\mathcal{E}_l$, nothing will be chosen and the reward is zero. And, in the other time steps, the expected reward will be $\dfrac{\sum_{i\in S_l}{v_ir_i}}{\sum_{i\in S_l}{v_i}}$. So, the expected cumulative reward in epoch $l$ is:
\begin{align}
\mathbb{E}\left(\sum_{t\in\mathcal{E}_l}{r_{c_t}}\right)&=\dfrac{\sum_{i\in S_l}{v_ir_i}}{\sum_{i\in S_l}{v_i}}\left(\mathbb{E}\left(|\mathcal{E}_l|\right)-1\right) \nonumber \\
								     &=\sum_{i\in S_l}{v_ir_i} \nonumber \ .
\end{align}
Then, 
\begin{align}
\mathbb{E}\left(\sum_{t\in \mathcal{E}_l}{\left(R\left(S_l, \mathbf{v}_l^{\rm UCB}\right)-r_{c_t}\right)}\right)&\leq\sum_{i\in S_l}{r_i\left(v_{i,l}^{\rm UCB}-v_i\right)}\nonumber\\
														&\leq 2\left(\sqrt{2}+\alpha\right)\sum_{i\in S_l}{r_i\sigma_{i,l}}\ .
\end{align}
\end{proof}
\subsection{Proof of Lemma \ref{lem:reward}}\label{sec:rewardproof}
\begin{proof}
Assume $[i_1,i_2,\cdots,i_{|S|}]$ is an arrange that $r_{i_k}\leq r_{i_l},\forall k\leq l$. As $r_{i_1}\geq R(S,\mathbf{v})$ and $u_{i_1}\geq v_{i_1}$, we have
\begin{align}
R(S,\mathbf{v})&=\dfrac{\sum_{i\in S}{v_ir_i}}{1+\sum_{i\in S}{v_i}} \nonumber \\
		        &\leq\dfrac{\sum_{i\in S}{v_ir_i}+(u_{i_1}-v_{i_1})r_{i_1}}{1+\sum_{i\in S}{v_i}+(u_{i_1}-v_{i_1})} \nonumber
\end{align}
, and 
\begin{align}
\dfrac{\sum_{i\in S}{v_ir_i}+(u_{i_1}-v_{i_1})r_{i_1}}{1+\sum_{i\in S}{v_i}+(u_{i_1}-v_{i_1})}&\leq r_{i_1} \nonumber
\end{align}
. By performing the process above iteratively, it can be proven that
\begin{align}
R(S,\mathbf{v})&\leq\dfrac{\sum_{i\in S}{v_ir_i}+\sum_{l\leq |S|}(u_{i_l}-v_{i_l})r_{i_l}}{1+\sum_{i\in S}{v_i}+\sum_{l\leq |S|}(u_{i_l}-v_{i_l})} 
		        &=R(S,\mathbf{u}) \nonumber
\end{align}
\end{proof}

\end{document}